\documentclass[a4paper,11pt]{article}

\usepackage[latin1]{inputenc}
\usepackage[english]{babel}
\usepackage{graphicx}
\usepackage{amssymb}
\usepackage{amsmath}%
\usepackage{amsfonts}%
\usepackage{amsthm}
\usepackage{url}

\newtheorem{theorem}{Theorem}

\newtheorem{corollary}{Corollary}

\newtheorem{example}{Example}
\newtheorem{lemma}{Lemma}

\newtheorem{remark}{Remark}
\numberwithin{equation}{section}
\DeclareMathOperator{\argmin}{argmin}
\DeclareMathOperator{\rank}{rank}
\DeclareMathOperator{\diag}{diag}

\renewcommand{\t}{{t}}
\renewcommand{\u}{{u}}
\renewcommand{\v}{{v}}

\title{On the Complexity of Robust PCA and \\ $\ell_1$-Norm Low-Rank Matrix Approximation} 
\author{
Nicolas Gillis\thanks{
Department of Mathematics and Operational Research, 
University of Mons, 
Rue de Houdain 9, 
7000 Mons, 
Belgium, 
 nicolas.gillis@umons.ac.be.}  \and 
Stephen A. Vavasis\thanks{Department of Combinatorics and Optimization, 
University of Waterloo,  
200 University Avenue W., 
Waterloo, ON N2L 3G1, Canada, 
vavasis@uwaterloo.ca.} 
}

\date{}
\begin{document}
\maketitle

\begin{abstract}
The low-rank matrix approximation problem with respect to the component-wise $\ell_1$-norm ($\ell_1$-LRA), which is closely related to robust principal component analysis (PCA), has become a very popular tool in data mining and machine learning. 
Robust PCA aims at recovering a low-rank matrix that was perturbed with sparse noise, with applications for example in foreground-background video separation. 
Although $\ell_1$-LRA is strongly believed to be NP-hard, there is, to the best of our knowledge, no formal proof of this fact. 
In this paper, we prove that $\ell_1$-LRA is NP-hard, already in the rank-one case, using a reduction from MAX CUT. 
Our derivations draw interesting connections between $\ell_1$-LRA and several other well-known problems, namely, robust PCA, 
$\ell_0$-LRA, binary matrix factorization, 
a particular densest bipartite subgraph problem, 
the computation of the cut norm of $\{-1,+1\}$ matrices, and the discrete basis problem, 
which we all prove to be NP-hard. 
\end{abstract}


\section{Introduction} 

Low-rank matrix approximation is a key problem in data analysis and machine learning. 
It is equivalent to linear dimensionality reduction that approximates a set of data points via a low-dimensional linear subspace. 
Given a matrix $M \in \mathbb{R}^{m \times n}$ and a factorization rank $r \leq \min(m,n)$, the problem can be stated as follows 
\begin{equation} 
\min_{X \in \mathbb{R}^{m \times n}} \quad ||M-X||  \quad \text{ such that } \quad \rank(X) \leq r, \nonumber 
\end{equation}
where $||.||$ is a matrix norm used to measure the error of the approximation.    
Equivalently, the matrix $X$ can be written as the outer product of two matrices and we have  
\begin{equation} 
\min_{U \in \mathbb{R}^{m \times r}, V \in \mathbb{R}^{r \times n}} \quad ||M-UV||  \, .   \nonumber 
\end{equation}  
Typically, the columns of the matrix $M$ represent $n$ data points in a $m$-dimensional space. 
The above decomposition gives $M(:,j) \approx \sum_{k=1}^r U(:,k) V(k,j)$ for all $j$, and hence is a linear and low-dimensional model for the data: the columns of $U$ are the basis of 
the linear subspace while each column of $V$ gives the coordinates in the basis $U$ to approximate each data point.  

The choice of the norm $||.||$ usually depends on the problem at hand and the noise model that is assumed on the input data. 
The most widely used norm is the Frobenius norm: 
\begin{equation} \label{LRA2} 
\min_{U \in \mathbb{R}^{m \times r}, V \in \mathbb{R}^{r \times n}} \quad ||M-UV||_F^2 = \sum_{i,j} (M-UV)_{ij}^2  \, , 
\end{equation} 
which assumes Gaussian noise. The problem~\eqref{LRA2} can be solved via the singular value decomposition (SVD); see \cite{GV96} and the references therein.  It is closely related to principal component analysis (PCA) as both problems are essentially equivalent (in PCA, the data is usually assumed to be mean centered). 
However, in some applications, other metrics might have to be used; here are two important examples: 
\begin{itemize}
\item \emph{Weights and missing data.} Adding weights in the objective function, that is, minimizing $\sum_{i,j} W_{ij} (M-UV)_{ij}^2$ 
for some nonnegative weight matrix  $W \in \mathbb{R}^{m \times n}_+$, allows to take into account different confidence levels among the entries of the input data $M$~\cite{GZ79}, or 
take into account missing entries (corresponding to zero entries of $W$). 
This has applications in machine learning for recommender systems~\cite{KBV09}, in computer vision to recover structure from motion~\cite{SIR95}, and in control for system identification~\cite{MU13, UM14}. 
However, the problem is NP-hard for any fixed factorization rank \cite{GG10c}, even in the rank-one case (that is, for $r= 1$).

\item \emph{Sparse input matrix.} If the input matrix is sparse, which is typical for example in applications involving large graphs and networks or document data sets,  Gaussian noise is not a good model and it makes more sense to minimize for example the (generalized) Kullback-Leibler divergence 
\[
D (M || UV)  = 
\sum_{i,j} \left( M_{ij} \log \left( \frac{M_{ij} }{(UV)_{ij}} \right) - M_{ij} + (UV)_{ij} \right) \, ;  
\]
see, e.g., the discussion in \cite{CK12} and the references therein.

\end{itemize}

Another important example that has been extensively studied is when the noise is sparse. 
In that case, the following problem is often considered: 
\begin{equation} \label{RPCA1} 
\min_{X \in \mathbb{R}^{m \times n}, S \in \mathbb{R}^{m \times n}} \; \rank(X) + \lambda ||S||_0  \quad \text{ such that } \quad M = X + S ,  
\end{equation}
where $||.||_0$ is the $\ell_0$ `norm' defined as 
\[
||S||_0 = \left| \{ (i,j) \ | \ S_{ij} \neq 0\ \} \right|
\]  
that counts the number of nonzero entries in the matrix $S$, and $\lambda > 0$ is a penalty parameter; see, e.g., \cite{WG09} and the references therein. Note that the equality constraint $M = X + S$ can be replaced with $||M-X-S|| \leq \epsilon$ in case some other type of noise is present (e.g., using the Frobenius norm allows to model both Gaussian and sparse noise). 
This problem is sometimes referred to as \emph{robust PCA}~\cite{CLM11}.   
Equivalently, if the rank of $X$ is fixed to $r$, the problem~\eqref{RPCA1} can be written as 
\begin{equation} \label{l0LRA} 
\min_{U \in \mathbb{R}^{m \times r}, V \in \mathbb{R}^{r \times n}} \; ||M-UV||_0 , 
\end{equation}
which we will refer to as $\ell_0$ low-rank matrix approximation ($\ell_0$-LRA). 

Several heuristic algorithms have been proposed for this problem. 
The two main families are the following: 
\begin{enumerate} 
\item \emph{Non-linear optimization-based algorithms}. 
Using the formulation \eqref{l0LRA} and replacing the $\ell_0$ norm by its well-known convex surrogate, the $\ell_1$ norm,  we have 
\begin{equation} \label{l1LRA} 
\min_{U \in \mathbb{R}^{m \times r}, V \in \mathbb{R}^{r \times n}} \; 
||M-UV||_1 = \sum_{i,j} |M_{ij} - (UV)_{ij}|,  
\end{equation}
which we will refer to as $\ell_1$-LRA. 
One can apply standard non-linear optimization schemes to~\eqref{l1LRA}, e.g., 
sequential rank-one updates~\cite{KK03}, 
alternating optimization~\cite{KK05} (a.k.a.\@ coordinate descent), 
the Wiberg algorithm~\cite{EVD10}, 
augmented Lagrangian approaches~\cite{ZLS12}, 
successive projections on hyperplanes and linear programming~\cite{BDB13}, to cite a few. The main drawback of this class of methods is that it does not guarantee to recover the global optimum of~\eqref{l1LRA} and is in general sensitive to initialization. (In fact, we will show that this problem is NP-hard; see Theorem~\ref{th3}.) 



\item \emph{Convexification}. 
Starting from the formulation~\eqref{RPCA1}, the standard convexification approach is to use the $\ell_1$ norm as a proxy for sparsity and the nuclear norm $||.||_*$ 
as a proxy for the rank function~\cite{CLM11}. 
The nuclear norm $||X||_* = \sum_i \sigma_i(X)$ is the sum of the singular values of $X$. 
Denoting $\sigma(X)$ the vector containing the singular values of $X$, we have $\rank(X) = ||\sigma(X)||_0$ and $||X||_* = ||\sigma(X)||_1$ which explains this choice of the nuclear norm: it is the $\ell_1$ norm of the vector of singular values. 
Finally, the difficult combinatorial problem~\eqref{RPCA1} is replaced with the following SDP-representable optimization problem (hence tractable)~\cite{CSPW11}: 
\begin{equation} \label{rpca} 
\min_{X, L} ||X||_* + \lambda ||L||_1 \quad \text{ such that } \quad M = X + L. 
\end{equation} 
Given that the input matrix $M$ has the sought structure (that is, sparse + low-rank) and 
satisfies some additional conditions (e.g., the non-zero entries of the sparse noise are not too numerous and appear randomly among the entries of $M$), 
solving~\eqref{rpca} guarantees to recover the sought solution~\cite{CSPW11, CLM11, XCS12}. 
This model has attracted a lot attention lately, 
both for its theoretical, algorithmic and application-oriented aspects. 

The two main drawback of this approach are that, 
\begin{enumerate}
\item if the input matrix is far from being a low-rank matrix plus sparse noise, 
the solution of~\eqref{rpca} can be very poor, and 
\item it requires to solve an optimization problem in $nm$ variables (for example, in foreground-background video separation, $n$ is the number of pixels and $m$ the number of frames, which can both be rather high). 
\end{enumerate}

More recently, a simple algorithm based on alternating projections 
(alternatively project onto the set of low-rank matrices and sparse matrices) was proved to recover the sought solutions under reasonable conditions~\cite{NNSA14} (similar to that of the convexification-based approaches). 
\end{enumerate}

Problems~\eqref{l0LRA},~\eqref{l1LRA},~\eqref{rpca} and variants have been used for many applications, e.g., 
foreground-background video separation, face recognition, latent semantic indexing, graphical modeling with latent variables, matrix rigidity and composite system identification; see the discussion in~\cite{CLM11, CSPW11, CVL14} 
and the references therein. 
It can also be used to identify large and dense subgraphs in bipartite graphs. 
Let $M$ be the biadjacency matrix of a graph representing the relationships between two groups of objects, e.g., movies vs.\@ users, documents vs.\@ words, or papers vs.\@ authors. 
Let us focus on the movies vs.\@ users example: each row of $M$ corresponds to a movie, each column to a user, and $M_{ij} = 1$ if and only if user $j$ has watched movie~$i$. 
Finding a subset of movies and a subset of users that is fully connected 
(referred to as a biclique) amounts to finding a community (a group of users watching the same movies). 
In the unweighted case, the matrix $M$ is binary and it can be easily checked that a community correspond to a rank-one binary matrix (a rectangle of ones). 
Moreover, in practice, some edges are often missing inside a community (all users have not watched all movies from their community) or some edges between communities might be present (some users might belong to several communities or watch movies from other communities).  
An important problem in this setting is to find the largest community. This can be cast as a rank-one robust PCA problem (see also Section~\ref{sec2}). Let us illustrate this with a simple example. 
\begin{example} \label{ex1} 
Assume we have a single community represented by the following matrix (the first three movies have been watched by the first four users):  
\[
M = \left( \begin{array}{ccccc} 
1 & 1 & 1 & 1 & 0 \\ 
1 & 1 & 1 & 1 & 0 \\ 
1 & 1 & 1 & 1 & 0 \\ 
0 & 0 & 0 & 0 & 0 \\
\end{array} \right) =  
\left( \begin{array}{c} 1 \\ 1 \\ 1 \\ 0 \\ 
\end{array} \right) 
\left( \begin{array}{ccccc} 1 & 1 & 1 & 1 & 0 \\ 
\end{array} \right). 
\] 
In real-world problems, some sparse noise is added to the matrix $M$ (see above). 
Under such perturbations, the optimal rank-one solution of $\ell_2$-LRA~\eqref{LRA2} loses the underlying structure very quickly, even when only a few entries in $M$ are modified:  
for example, adding three edges gives the following optimal rank-one approximation 
\[
\tilde{M} = \left( \begin{array}{ccccc} 
1 & 1 & 1 & 1 & 0 \\ 
1 & 1 & 1 & 1 & 1 \\ 
1 & 1 & 1 & 1 & 0 \\ 
1 & 0 & 0 & 0 & 1 \\
\end{array} \right) \approx 
\left( \begin{array}{ccccc}  
  1.03 & 0.92 & 0.92 & 0.92 & 0.44 \\
  1.15 & 1.02 & 1.02 & 1.02 & 0.50\\
  1.03 & 0.92 & 0.92 & 0.92 & 0.44\\
  0.40 & 0.36 & 0.36 & 0.36 & 0.17\\ 
\end{array} \right)  , 
\] 
while the optimal rank-one solution\footnote{You can run this example with our code available from \url{https://sites.google.com/site/nicolasgillis/code}.} of \eqref{l0LRA} and \eqref{l1LRA} is given by $M$. 
Ames and Vavasis~\cite{AV11}, and Doan and Vavasis~\cite{DV13} studied this particular variant of robust PCA (although they did not call it that) and showed that the convexification approach based on the nuclear norm is able to recover the largest community given that sufficiently few edges are perturbed (either randomly or by an adversary). 
\end{example}

Another closely related class of low-rank matrix approximation problems has also attracted much attention lately, namely 
\[
\min_{X, \rank(X) = r} \; \sum_{i = 1}^n  \big( ||M(:,i) - X(:,i)||_2 \big)^p . 
\] 
For $p=2$, this is $\ell_2$-LRA~\eqref{LRA2}. 
For $1 \leq p \neq 2$, the problem has been shown to be NP-hard~\cite{GRS12, CW15}; 
and approximation algorithms have been proposed; see~\cite{CW15} and the references therein.

\subsection{Contribution and outline of the paper}

Although robust PCA and its variants are widely believed to be NP-hard (see, e.g.,~\cite{K08, WG09}), it has, to the best of our knowledge, never been proved rigorously.  
In this paper, we prove that $\ell_0$-LRA and $\ell_1$-LRA are NP-hard, already in the rank-one case, that is, for $r=1$. 
This solves the first part of the open question~2 in~\cite{W14}.  

In section~\ref{sec2}, we focus on $\ell_0$-LRA~\eqref{l0LRA} of a binary matrix, which we show is equivalent to rank-one binary matrix factorization (BMF). We also show the connection with the problem of finding a large and dense subgraph in a bipartite graph.  
In section~\ref{sec3}, we prove that rank-one BMF is equivalent to the cut norm computation of $\{-1,+1\}$ matrices, which we prove to be NP-hard using an equivalence with the computation of the norm $||.||_{\infty \rightarrow 1}$ and a reduction from MAX CUT (Theorem~\ref{th1}). 
This implies that $\ell_0$-LRA and rank-one BMF are both NP-hard. 
In section~\ref{sec4}, we prove that, for a $\{-1,+1\}$ input matrix, any optimal solution of rank-one $\ell_1$-LRA~\eqref{l1LRA} can be transformed into a rank-one solution with entries in $\{-1,+1\}$ (Theorem~\ref{th2}). 
We also show that, for $\{-1,+1\}$ matrices, rank-one $\ell_1$-LRA is equivalent to the computation of the norm $||.||_{\infty \rightarrow 1}$  which implies NP-hardness of $\ell_1$-LRA (Theorem~\ref{th3}). 
In section~\ref{sec5}, we briefly describe how the complexity results in the rank-one case can be generalized to higher ranks.


\section{Binary matrix factorization and densest bipartite subgraph} \label{sec2}

Let $M \in \{0,1\}^{m \times n}$ be a binary matrix. Rank-one BMF is the problem 
\begin{equation} \label{r1BMF}
\min_{ u \in \{0,1\}^{m}, v \in \{0,1\}^{n} } \; ||M-uv^T||. 
\end{equation} 
BMF was used successfully to mine discrete patterns with applications for example to analyze gene expression data~\cite{SJY09, SLD10, MGT15}. We refer the reader to the tutorial~\url{http://people.mpi-inf.mpg.de/~pmiettin/bmf_tutorial} and the references therein for more details.  Although BMF is conjectured to be NP-hard~\cite{SJY09, MGT15}, there is, to the best of our knowledge, no formal proof of this fact. We will prove in this paper that it is in fact NP-hard. 

For BMF, all component-wise norms, that is, all norms of the form $||M-uv^T|| = \sum_{i,j} f(M_{i,j}, u_i v_j)$ for some function $f$ 
with $f(z, z) = 0$ and $f(z, z') > 0$ for $z' \neq z$, are equivalent since both $M$ and $uv^T$ are binary matrices. 
For such norms, $||M-uv^T||$ amounts to count the number of mismatches between $M$ and $uv^T$ and hence rank-one BMF can be formulated as follows 
\begin{equation} \label{bmf}
\min_{u \in \{0,1\}^{m}, v \in \{0,1\}^{n}} ||M-uv^T||_0 . 
\end{equation}

The matrix $M \in \{0,1\}^{m \times n}$ can be interpreted as the biadjacency of a bipartite graph $G = (S \times T, E)$ 
with 
$S = \{s_1,s_2,\dots, s_m\}$, 
$T = \{t_1,t_2,\dots, t_n\}$, 
and 
$E \subset S \times T$ where $M_{ij} = 1 \iff (s_i,t_j) \in E$. 
Let us denote $E(S',T')$ the number of edges in $G$ in the subgraph induced by $S' \times T'$, and denote $\bar{S'} = S \backslash S'$. 
Then \eqref{bmf} is the problem of finding two subsets $S' \subseteq S$ and $T' \subseteq T$ (with $s_i \in S' \iff u_i = 1$ and $t_j \in V'_1 \iff v_j = 1$), such that the subsets of vertices $S'$ and $T'$ maximize the following quantity 
\[ 
 \underbrace{E(S',T')}_{\text{\# edges in $S' \times T'$}}  
- \underbrace{ ( |S'|  |T'|  - E(S',T') )}_{\text{\# non-edges in $S' \times T'$}}
- \underbrace{ ( |E| - E(S',T') )  }_{\text{\# edges outside $S' \times T'$}}  
  = 3 E(S',T') - |S'|  |T'| - |E| . 
\] 
This problem is a particular variant of the general problem of finding large dense subgraphs in bipartite graphs; see, e.g., \cite{AHI02, K06, KS09} and the references therein. 
If the size of the subgraph is fixed a priori, finding the densest subgraph is NP-hard~\cite{AHI02}, even to approximate~\cite{K06}. 
However, finding a partition $S' \times T'$ that maximizes $\frac{E(S',T')}{\sqrt{|S'| |T'|}}$ can be done in polynomial time~\cite{KV99}.  
The problem above is slightly different because we do not fix the size nor try to find the densest subgraph: it looks for a subgraph that is at the same time large and relatively dense. 
Hence, as far as we know, the complexity results for the densest subgraph do not apply to our problem (at least we could not find a  reduction from these problems to ours).

In the following we prove that rank-one $\ell_0$-LRA is equivalent to rank-one BMF. Let us show the following straightforward lemma. 
\begin{lemma} \label{lem1}
Let $x \in \mathbb{R}^{m}$ and $y \in \mathbb{R}^{n}$, and let $M$ be a binary matrix.   
Applying the following simple transformation to $x$ and $y$ 
\[
\Phi(x)_i = 
\left\{ 
\begin{array}{cc} 
0 & \text{ if $x_i$ = 0,  } \\ 
1 & \text{ otherwise, } \\
\end{array} \right.
\]
gives 
\[
||M-\Phi(x) \Phi(y)^T||_0 \leq ||M-xy^T||_0 . 
\] 
\end{lemma}
\begin{proof}
There are two cases 
\begin{enumerate} 
\item If $x_i y_j = 0$, then $\Phi(x_i) \Phi(y_j) = 0$ hence the transformation does not affect the approximation. 

\item If $x_i y_j \neq 0$, then $\Phi(x_i) \Phi(y_j) = 1$. If $M_{ij} = 0$ then $||M_{ij}-x_iy_j||_0 = ||M-\Phi(x_i) \Phi(y_j)||_0 = 1$ while, if $M_{ij} = 1$, $||M_{ij}-x_iy_j||_0 \geq ||M-\Phi(x_i) \Phi(y_j)||_0 = 0$. 
\end{enumerate} 
\end{proof} 

\begin{corollary} \label{cor1} 
For a binary input matrix $M$, rank-one $\ell_0$-LRA is equivalent to rank-one BMF. 
\end{corollary}

In the next section, we prove that rank-one $\ell_0$-LRA of a binary matrix is NP-hard, 
showing it is equivalent to the computation of the cut norm of a $\{-1,+1\}$ matrix which we show is NP-hard 
using a reduction from MAX CUT.

\section{Cut norm of $\{-1,+1\}$ matrices and rank-one $\ell_0$-LRA} \label{sec3}

Given a matrix $M$, its cut norm is defined as~\cite{FK99} 
\begin{equation} \label{cutnorm}
||M||_C 
\; = \; \max_{u \in \{0,1\}^m, v \in \{0,1\}^n} 
\left|  u^T M v \right| . 
\end{equation} 
The fact that this is a norm on $\mathbb{R}^{m\times n}$ (regarded as a vector space isomorphic to $\mathbb{R}^{mn}$) can be checked easily. 
In~\cite{FK99}, Frieze and Kannan study the low-rank matrix approximation problem with respect to the cut norm and design an algorithm that provides a solution which is the sum of $O(1/\epsilon^2)$ rank-one matrices with error at most $\epsilon \, mn$. 

In~\cite{AN06}, Alon and Naor prove NP-hardness of the problem of computing the cut norm using a reduction from MAX~CUT. The reduction uses matrices $M$ with entries in $\{-1,0,+1\}$.  

Let us show how the cut norm computation is related to low-rank matrix approximations. Let $M$ be a binary matrix and let us derive some equivalent forms of rank-one BMF or, equivalently, to rank-one $\ell_0$-LRA (Corollary~\ref{cor1}). First note that for any $(u,v)$, we have 
\begin{equation} \label{simpleLA}
||M-uv^T||_F^2 = ||M||_F^2  - 2 \sum_{i,j} M_{ij} u_i v_j + \sum_{i,j} (u_i v_j)^2. 
\end{equation} 
Hence, we obtain 
\begin{align*}
\min_{u \in \{0,1\}^m, v \in \{0,1\}^n} ||M-uv^T||_F^2
& =  ||M||_F^2 + \min_{u \in \{0,1\}^m, v \in \{0,1\}^n}  \sum_{i,j} u_i v_j - 2 \sum_{i,j} M_{ij} u_i v_j  \\ 
& =  ||M||_F^2 + \min_{u \in \{0,1\}^m, v \in \{0,1\}^n}  \sum_{i,j} (1-2M_{ij}) u_i v_j  \\ 
& =  ||M||_F^2 + \max_{u \in \{0,1\}^m, v \in \{0,1\}^n}  \sum_{i,j} (2M_{ij}-1) u_i v_j  . 
\end{align*}
The last problem is closely related to the cut norm of the $\{-1,+1\}$ matrix $A = 2M - 1$ and $-A$. In fact,
\[
||A||_C = \max \left( 
\max_{u \in \{0,1\}^m, v \in \{0,1\}^n}  \sum_{i,j} (2M_{ij}-1) u_i v_j, 
\max_{u \in \{0,1\}^m, v \in \{0,1\}^n}  \sum_{i,j} (1-2M_{ij}) u_i v_j \right) . 
\] 
Hence, if we were able to solve $\max_{u \in \{0,1\}^m, v \in \{0,1\}^n}  u^T X v$ for any $X \in \{-1,+1\}^{m \times n}$, we would be able to compute the cut norm of $A$. 
However, the cut norm problem was shown to be NP-hard only for $A \in \{-1,0,+1\}^{m \times n}$ using a reduction from MAX CUT. It turns out that the reduction no longer holds for  $A \in \{-1,+1\}^{m \times n}$. 

In the following, we prove that computing the cut norm of $\{-1,+1\}^{m \times n}$ matrices is NP-hard hence rank-one BMF and rank-one $\ell_0$-LRA are also NP-hard. 

First, let us consider the following norm introduced in~\cite{AN06}: for $A \in \mathbb{R}^{m \times n}$  
\begin{equation} \label{inf1norm}
|| A ||_{\infty \rightarrow 1} \; = \; 
\max_{u \in \{-1,+1\}^m,v \in \{-1,+1\}^n} u^T A v. 
\end {equation} 
To the best of our knowledge, solving~\eqref{inf1norm} was first shown to be NP-hard in~\cite{PR93} 
(for matrices in $\{0,1,p\}$ for some positive real $p$); see also~\cite{R2000}. 

This norm is closely related to the cut norm, in fact, it can be easily shown that~\cite{AN06} 
\[
|| A ||_{C}
\; \leq  \;  
|| A ||_{\infty \rightarrow 1}  
\; \leq  \; 
4 || A ||_{C} . 
\] 
Moreover, 
\begin{lemma} \label{lem2}
Given $A \in \mathbb{R}^{m \times n}$, we have 
\[
|| A ||_{\infty \rightarrow 1} 
\; 
= 
\; 
\left\| 
\left( 
\begin{array}{cc}
A & -A \\
-A & A 
\end{array} \right) 
\right\|_C . 
\]
\end{lemma}
\begin{proof}
It was proved in~\cite{AN06} that a matrix $B$ whose rows and columns sum to zero satisfies
\[
|| B ||_{\infty \rightarrow 1} 
\; 
= 
\; 
4 || B ||_C. 
\] 
In fact, let $e$ be the vector of all ones, and $u = 2x - e$ and $v = 2y-e$ have entries in $\{-1,+1\}$ where $x$ and $y$ have binary entries. Then, 
\[
u^T B v 
= (2x - e)^T B (2y-e) 
= 4 x^T B y - 2x^T Be - 2 e^T B y + e^T B e
= 4 x^T B y
\]
since $Be = 0$ and $B^T e = 0$ by assumption. 

Since the rows and columns of the matrix  
\[ 
B = \left( 
\begin{array}{cc}
A & -A \\
-A & A 
\end{array} \right)
\]
sum to zero, we have  $|| B ||_{\infty \rightarrow 1} 
= 4 || B ||_C$. To conclude the proof, we show that 
$|| B ||_{\infty \rightarrow 1} = 4 || A ||_{\infty \rightarrow 1}$. 
Let $u = [u_1; u_2]$ and $v = [v_1;v_2]$ be a solution of~\eqref{inf1norm} for $B$. We have 
\[
|| B ||_{\infty \rightarrow 1}
= u^T B v 
= u_1^T A v_1 - u_1^T A v_2  + u_2^T  A v_2  - u_2^T A v_1 . 
\]
Clearly, each term has to be smaller than $|| A ||_{\infty \rightarrow 1}$ since $u$ and $v$ have entries in $\{-1,+1\}$ and $|| A ||_{\infty \rightarrow 1} = || -A ||_{\infty \rightarrow 1}$. However, taking $u_1$ and $v_1$ such that 
\[
|| A ||_{\infty \rightarrow 1} = u_1^T A v_1, 
\]
 and $u_2 = -u_1$,  $v_2 = -v_1$ gives 
$|| B ||_{\infty \rightarrow 1} = 4 || A ||_{\infty \rightarrow 1}$. 
\end{proof}

Let us show that computing the norm $||.||_{\infty \rightarrow 1}$ is NP-hard for $\{-1,+1\}$ matrices, 
which will imply, by Lemma~\ref{lem2}, that it is NP-hard to compute the cut norm of $\{-1,+1\}$ matrices. 

\begin{theorem} \label{th1}
It is NP-hard to compute the norm $||.||_{\infty \rightarrow 1}$ for $\{-1,+1\}$-matrices. 
\end{theorem} 
\begin{proof}
The problem under consideration is, given an $m \times n$ matrix $A$ all of whose
entries are $\pm 1$, find $\u\in\{-1,1\}^m$ and $\v\in\{-1,1\}^n$ to maximize
$\u^TA\v$.  Let us define P1 the decision version of this problem: 
\begin{quote} 
(P1) Given $A$ and an integer $d^*$, does there exist  
$\u\in\{-1,1\}^m$ and $\v\in\{-1,1\}^n$ 
such that $\u^TA\v\ge d^*$? 
\end{quote} 
We prove that this problem is NP-hard by a reduction from 
unweighted MAX CUT. 
\begin{quote} 
(MAX CUT) Given a graph $G=(V,E)$ and a positive integer $c^*$, is there a cut containing at least $c^*$ edges? 
\end{quote}  
Given an instance $(G,c^*)$ of MAX CUT,  let us produce an instance $(A,d^*)$ of P1 as follows.
Let $p\ge 1$ be an integer to be determined later.  The matrix $A$ has dimension $m \times n$
with $m=p|E|$ and $n=p|V|$. 
It will be constructed via  
$p\times p$ blocks as follows.  Suppose edge $q \in E$, $1\le q \le |E|$, has endpoints
$(i,j)$, $1\le i<j\le |V|$.  Then the $(q,i)$ block of $A$ 
is the $p\times p$ block of all 1's; the $(q,j)$ block of $A$ is the $p\times p$
block of all $-1$'s, and the $(q,l)$ block of $A$ for any $l\in\{1,\ldots,|V|\}-\{i,j\}$
is the $p\times p$ Hadamard matrix $H$.

Recall that a $p\times p$ {\em Hadamard matrix} $H$ is a matrix all of
whose entries are $\pm 1$ and such that the columns are mutually
orthogonal.  In the case that $p$ is a power of 2, there is a
straightforward recursive construction of a Hadamard matrix: the
$1\times 1$ Hadamard matrix is $[\,1\,]$, and the $p\times p$ Hadamard
matrix is $[H_0,H_0;-H_0,H_0]$, where $H_0$ is the $(p/2)\times (p/2)$
Hadamard matrix.  For the remainder of this proof, assume that $p$ is
a power of 2 so that we can rely on this simple construction.

Suppose $(S,\bar S)$ is a partition of $V$ (i.e., a cut).  We can associate
vectors $(\u,\v)$ with $S$ as follows: for $i\in S$, let the $i$th block of
$\v$ contain $p$ 1's.  For $i\in \bar S$, let the $i$th block of $\v$
contain $p$ $-1$'s.  For $q=(i,j)\in E$, $i<j$, 
such that $i\in S$, $j\in \bar S$, let the $q$th block of $\u$ contain $p$ 1's.
For $(i,j)\in E$, $i<j$, such that $i\in \bar S$, $j\in S$, let the $q$th
block of $\u$ contain $p$ $-1$'s.  Finally, for $q=(i,j)$ such that
$\{i,j\}\subset S$ or $\{i,j\}\subset \bar S$, the $q$th block of $\u$ may be
selected arbitrarily. 

For the choice of $(\u,\v)$ in the last paragraph, let us obtain a lower bound
on $\u^TA\v$.  Observe that the signs have been chosen such that for each $q=(i,j)$
that crosses the cut, the blocks of $A$ indexed $(q,i)$ and $(q,j)$ contribute
$2p^2$ to $\u^TA\v$.  Blocks of $A$ indexed $(q,i)$ and $(q,j)$ such that
$q=(i,j)$ does not cross the cut contribute 0.  Finally, we have to account
for the blocks of the form $(q,l)$ where $l$ is not an endpoint of $q$.  For this,
we make the following observation:  if $\u_0,\v_0$ are two $\pm 1$ $p$-dimensional vectors 
and $H$ is a $p\times p$ Hadamard matrix, then $|\u_0^TH\v_0|\le p^{3/2}$ \cite{BS71}.  This
follows because $\Vert\u_0\Vert_2=\Vert\v_0\Vert_2=\Vert H\Vert_2 =\sqrt{p}$.  (The
last equation follows because the $\ell_2$-norm of an orthogonal matrix is exactly 1, 
and a Hadamard matrix is an orthogonal
matrix scaled by $\sqrt{p}$.)

Thus, a lower bound on the objective $\u^TA\v$ for the cut $(S,\bar S)$ is
$2p^2c-|E||V|p^{3/2}$, where $c$ is the size of the cut induced by $(S,\bar S)$ because
there at most $|E||V|$ blocks of the form $(q,l)$ where $l$ is not an endpoint of $q$.
 So the decision problem posed for P1 is:
``Given $A$ constructed above,
is the objective function for this $A$ at least $d^*=2p^2c^*-|E||V| p^{3/2}$?''
We have already shown that if there is a cut of size $c^*$ in the graph, then
there is a solution of size $d^*$ for P1.

The last thing to prove is that if the max cut of $G$ has fewer than $c^*$
edges, then P1 is a no-instance, i.e., for any
$\u\in\{-1,1\}^{|E|p}$ and $\v\in\{-1,1\}^{|V|p}$,
$\u^TA\v < d^*$.  Choose any
$\u\in\{-1,1\}^{|E|p}$ and $\v\in\{-1,1\}^{|V|p}$. Let $s_q$ denote the number
of $1$'s in block $q$ of $\u$, $q=1,\ldots |E|$ (so that $p-s_q$ is the number of
$-1$'s in the block).  Let $t_i$ denote the number of $1$'s in block $i$ of
$\v$.  It is straightforward to show that contribution to the objective function
from the $2|E|$ blocks of $A$ that correspond to edges (i.e., the blocks
numbered $(q,i)$ where $i$ is an endpoint of $q$)
is precisely $T_1$, where
$$
T_1=\sum_{q=(i,j)\in E}2(t_i-t_j)(2s_q-p).
$$
Now observe that 
$$
T_1 \le  \sum_{q=(i,j)\in E} 2|t_i-t_j|p
$$
since the second factor in the previous summation has absolute value at
most $p$.  Next, notice that this latter summation, regarded as a function of
$\t\in [0,p]^{|V|}$, is maximized at an extreme point because it is a convex
function.  Therefore, there exists a vector $\tilde \t\in\{0,p\}^{|V|}$ such that
$$
T_1 \le 2p \sum_{q=(i,j)\in E}| \tilde t_i-\tilde t_j|.
$$
The latter is exactly $2p^2$ multiplied by the size of cut induced by $\tilde t$
(i.e., the cut $(S,\bar S)$ with $i\in S$ if and only if $\tilde t_i=p$).  Since we
are considering the case that all cuts have fewer than $c^*$ edges, 
$$
T_1 \le 2p^2(c^*-1).
$$ 
This accounts for the $2|E|$ blocks that correspond to edges.
For the $|E||V|-2|E|$ blocks that do not correspond to edges, the contribution
to the objective function is at most $p^{3/2}$ for the same reason as above.
Therefore,
$$
\u^TA\v \le 2p^2(c^*-1)+|E||V|p^{3/2} = 2p^2c^* - (2p^2 - |E||V|p^{3/2}).
$$
We see that the right-hand side is less than $d^*=2p^2c^*-|E||V|p^{3/2}$ provided that $2|E||V|p^{3/2}<2p^2$,
i.e., $\sqrt{p}>|E||V|$.  Therefore, we choose $p>|E|^2|V|^2$ and also $p$ a power of 2. 
\end{proof}

\begin{corollary}
It is NP-hard to compute the cut norm~\eqref{cutnorm} of $\{-1,+1\}$-matrices. 
\end{corollary}

\begin{corollary}
Rank-one $\ell_0$-LRA, that is, problem~\eqref{l0LRA} with $r=1$, is NP-hard. 
\end{corollary} 

\begin{corollary}
Rank-one BMF~\eqref{r1BMF} is NP-hard. 
\end{corollary}

\begin{remark}[The discrete basis problem and boolean matrix factorization] 
The discrete basis problem (DSP)~\cite{MMG08}, also known as Boolean matrix factorization, is similar to BMF and can be formulated a follows: 
given a binary matrix $M \in \{0,1\}^{m \times n}$ and a rank $r$, solve
\begin{equation} \label{dsp}
\min_{U \in \{0,1\}^{m \times r}, V \in \{0,1\}^{r \times n}} || M - U \circ V ||
\end{equation}
where $(U \circ V)_{ij} = \bigoplus_k U_{ik} V_{kj}$ where $0 \bigoplus 0 = 0$, $0 \bigoplus 1 = 1$ and $1 \bigoplus 1 = 1$. 
For $r = 1$, BMF and DSP coincide, therefore our result also implies that rank-one DSP is NP-hard. 


Note that DSP is closely related to the rectangle covering problem which is equivalent to the minimum biclique cover problem in bipartite graph;  
see Fiorini et al.~\cite{FK11}. 
\end{remark}

\section{The component-wise $\ell_1$ low-rank matrix approximation problem} \label{sec4}

In this section, we prove that for a $\{-1,+1\}$ matrix, any optimal solution of rank-one $\ell_1$-LRA~\eqref{l1LRA} 
can be assumed to have entries in $\{-1,+1\}$ (Theorem~\ref{th2}). Moreover, we prove that computing the norm $||.||_{\infty \rightarrow 1}$ of a $\{-1,+1\}$ matrix is equivalent to solving rank-one $\ell_1$-LRA for that matrix (Lemma~\ref{lem3}). 
This will imply that rank-one $\ell_1$-LRA is NP-hard (Theorem~\ref{th3}). 

\begin{lemma} \label{lem3}
For $A \in \{-1,+1\}^{m \times n}$, computing $||A||_{\infty \rightarrow 1}$ is equivalent to solving 
\begin{equation} 
\min_{x \in \{-1,+1\}^{m}, y \in \{-1,+1\}^{n}} ||A - xy^T||_1 . \nonumber 
\end{equation}
\end{lemma}
\begin{proof}
 We have 
\begin{align*}
2 \min_{x \in \{-1,+1\}^{m}, y \in \{-1,+1\}^{n}} ||A - xy^T||_1 
& =  \min_{x \in \{-1,+1\}^{m}, y \in \{-1,+1\}^{n}} ||A - xy^T||_F^2  \\ 
& =  ||A||_F^2 + \min_{x \in \{-1,+1\}^{m}, y \in \{-1,+1\}^{n}}  -2 x^T A y + ||xy^T||_F^2  \\ 
& =  ||A||_F^2 + mn + 2 ||A||_{\infty \rightarrow 1} .  
\end{align*}
The first equality follows from the fact that $A-uv^T$ has entries in $\{-2,0,+2\}$, 
the second from~\eqref{simpleLA}, and the third by definition of  $||A||_{\infty \rightarrow 1}$ 
and since $xy^T$ has entries in $\{-1,+1\}$.  
\end{proof}

\begin{theorem} \label{th2}
Let $A$ be a $\{-1,+1\}$ matrix, 
then any optimal solution of rank-one $\ell_1$-LRA~\eqref{l1LRA} for input matrix $A$ can be transformed into an optimal solution whose entries are in $\{-1,+1\}$. This implies that 
\[
\min_{x \in \{-1,+1\}^{m}, y \in \{-1,+1\}^{n}} ||A - xy^T||_1  
\quad = \quad  
\min_{x \in \mathbb{R}^{m}, y \in \mathbb{R}^{n}} ||A - xy^T||_1 . 
\] 
\end{theorem}
\begin{proof}
Let $(u,v)$ be an optimal solution of rank-one $\ell_1$-LRA for matrix $A$, that is, of 
\[
\min_{x \in \mathbb{R}^{m}, y \in \mathbb{R}^{n}} ||A - xy^T||_1 = \sum_{i,j} |A_{ij}-x_iy_j| . 
\] 
First note that $u \neq 0$ and $v\neq0$ since we can approximate exactly at least one entry of $A$, e.g., take $u_1=1$, $v_1=A_{11}$ and all other entries of $u$ and $v$ equal to zero. 

By optimality, we have for all $1 \leq i \leq m$ that 
\[
u_i 
\quad = \quad 
\argmin_{x_i \in \mathbb{R}} \sum_{j} |A_{ij}-x_i v_j| 
\quad = \quad   
\argmin_{x_i \in \mathbb{R}} \sum_{j, v_j \neq 0} |A_{ij}-x_i v_j| . 
\] 
The objective function of this problem is piece-wise linear hence we can assume without loss of generality (w.l.o.g.) that $u_i$ is equal to one of the break points (otherwise we can easily modify $u$ so that this property holds), 
that is, there exists $j$ such that $v_j \neq 0$ and 
\[
u_i = \frac{A_{ij}}{v_j} \in \left\{ \frac{1}{v_j} , \frac{-1}{v_j} \right\} . 
\]
By symmetry, the same holds of $v$. Note that this implies that we can assume w.l.o.g. that $u$ and $v$ have all their entries different from zero. 

Let the entries of $u$ take $k$ different values in absolute value $0 < \alpha_1 < \alpha_2 < \dots < \alpha_k$. Note that the entries of $v$ also take $k$ different values in absolute value, namely $1/\alpha_p$ for $1 \leq p \leq k$. 

We partition the solution $(u,v)$ into $k$ blocks $(K^u_p, K^v_p)$ for $1 \leq p \leq k$ 
such that for all $i \in K^u_p$ and $j \in K^v_p$ 
\[
u_i = \pm \alpha_p \quad \text{ and } \quad v_j  = \pm \frac{1}{\alpha_p} .   
\]  
Let us assume w.l.o.g.\@ that 
\begin{itemize} 
\item $u$ and $v$ are scaled such that $\max_i |u_i| = 1$ hence $\min_j |v_j| = 1$. This implies $\alpha_k= 1$. 
\item we permute the entries of $u$ (and the corresponding rows of $A$) such that the entries of $|u|$ are in nondecreasing order, and we permute the entries of $v$ (and the corresponding rows of $A$) such that the entries of $|v|$ are in nonincreasing order. 
\end{itemize} 
Therefore, after suitable permutations and scaling, the rank-one solution can be put, w.l.o.g., in the following form 
\[
u = \left( \pm \alpha_1, \dots, \pm \alpha_1, \pm \alpha_2, \dots, \pm \alpha_2, \dots, 
\pm \alpha_{k-1}, \dots, \pm \alpha_{k-1}, \pm 1, \dots, \pm 1 \right)^T,
\]
and 
\[
v = \left( \pm \alpha_1^{-1}, \dots, \pm \alpha_1^{-1}, \pm \alpha_2^{-1}, \dots, \pm \alpha_2^{-1}, \dots, 
\pm \alpha_{k-1}^{-1}, \dots, \pm \alpha_{k-1}^{-1}, \pm 1, \dots, \pm 1 \right)^T,
\]
which gives the following rank-one solution: 
\begin{center}
\begin{tabular}{c||c c c | c | c c c | c c c}
   &  $\pm \alpha_1^{-1}$ & $\dots$ & $\pm \alpha_1^{-1}$  
	& $\dots$ &  $\pm \alpha_{k-1}^{-1}$ &  \dots & $\pm \alpha_{k-1}^{-1}$  & $\pm 1$ & $\dots$ & $\pm 1$ \\ \hline  \hline  
$\pm \alpha_1$ & & & & & & & & & &  \\ 
$\vdots$ & & $\pm$ \bf 1 & & \dots &  & $\pm \frac{ \alpha_1}{ \alpha_{k-1}}$ \bf 1 &  & & $\pm  \alpha_1$ \bf 1 &  \\ 
$\pm \alpha_1$ & & & & & & & & &  & \\ 
\hline 
$\vdots$ & & \vdots & &  &  & \vdots& & & \vdots &  \\ 
\hline  
$\pm \alpha_{k-1}$ & & & & & & & & & &  \\ 
$\vdots$ & &  $\pm \frac{ \alpha_{k-1}}{ \alpha_1}$ \bf 1 & & \dots &  & $\pm$ \bf 1 &   & & $\pm \alpha_{k-1}$ \bf 1 &  \\ 
$\pm \alpha_{k-1}$ & &  & & & & & & & &  \\  
\hline 
$\pm 1$ & & & & & & & & & &  \\ 
$\vdots$ & &  $\pm \frac{1}{ \alpha_1}$ \bf 1 & & \dots &  & $\pm \frac{1}{ \alpha_{k-1}}$ \bf 1 &  & & $\pm$ \bf 1 &  \\ 
$\pm 1$ & &  & & & & & & & &  \\  
	\end{tabular}
\end{center} 
\noindent where $\bf 1$ is the matrix of all ones of appropriate dimensions.  
The rank-one matrix $uv^T$ is a block matrix with $k^2$ blocks at positions $(p,q)$ for $1 \leq p,q \leq k$. 
Note that the diagonal blocks are not necessarily square, that is, $|u|$ may contain more (or less) entries equal to $\alpha_p$ than $|v|$ contains entries equal to $\alpha_p^{-1}$. 
Note also that the blocks above (resp.\@ below) the diagonal 
have entries strictly smaller (resp.\@ larger) than one in absolute value since 
$\alpha_1 < \alpha_2 < \dots < \alpha_k = 1$. \\ 

Now, let us consider two modifications of this rank-one solution and let us see how the objective function of rank-one $\ell_1$-LRA~\eqref{l1LRA} changes with these modifications. 

\paragraph{Move 1.} We divide the entries in $u$ different from $\pm 1$ by $\alpha_{k-1}$ and multiply the entries in $v$ different from $\pm 1$ by $\alpha_{k-1}$ from which we get the following solution $(u',v')$:  
\[
u' = \left( \pm \frac{\alpha_1}{\alpha_{k-1}}, \dots, \pm \frac{\alpha_1}{\alpha_{k-1}}, \pm \frac{\alpha_2}{\alpha_{k-1}}, \dots, \pm \frac{\alpha_2}{\alpha_{k-1}}, \dots, 
\frac{\alpha_{k-2}}{\alpha_{k-1}}, \dots,  \frac{\alpha_{k-2}}{\alpha_{k-1}}, \pm 1, \dots, \pm 1 \right)^T,
\]
and 
\[
v' = \left( \pm \frac{\alpha_{k-1}}{\alpha_1}, \dots, \pm  \frac{\alpha_{k-1}}{\alpha_1}, 
\pm  \frac{\alpha_{k-1}}{\alpha_2}, \dots, \pm  \frac{\alpha_{k-1}}{\alpha_2}, \dots, 
\pm  \frac{\alpha_{k-1}}{\alpha_{k-2}}, \dots, \pm  \frac{\alpha_{k-1}}{\alpha_{k-2}}, 
\pm 1, \dots, \pm 1 \right)^T. 
\]
Compared to $uv^T$, only the blocks at position $(p,k)$ and $(k,p)$ for $1 \leq p \leq k-1$  have been modified by the changes from $(u,v)$ to $(u',v')$.   
Let us denote $\delta^{(1)}_p$ the modification of the objective function for the two blocks $(p,k)$ and $(k,p)$ (the 1 stands for move 1), that is, 
\[
\delta^{(1)}_p \quad = 
\sum_{(i,j) \text{ in $(p,k),(k,p)$-blocks}} \left( |M_{ij}-u'_i v'_j| 
- |M_{ij}-u_i v_j| \right) .  
\] 
Let us denote $a_p$ (resp.\@ $b_p$) 
the number of entries in the $(p,k)$-block  such that the sign of $A$ does (resp.\@ does not) match the sign in $uv^T$.   
Let us also denote $c_p$ (resp.\@ $d_p$) 
the number of entries in the block $(k,p)$ such that the sign of $A$ and the sign of $uv^T$ does (resp.\@ does not) match. 
To simplify notations, let us denote $\beta = \alpha_{k-1}$ where $1 > \beta \geq \alpha_p$ for all $1 \leq p \leq k-1$ by construction.  

In the block $(p,k)$ ($1 \leq p \leq k-1$) the error of the solution $uv^T$ is given by 
\[
e_{(p,k)} = a_p (1-\alpha_p) + b_p (1+\alpha_p), 
\]
while the error of the solution $u'v'^T$ is given by 
\[
e'_{(p,k)} = 
a_p \left(1-\frac{\alpha_p}{\beta}\right) 
+ 
b_p \left( 1+\frac{\alpha_p}{\beta} \right). 
\]
The difference is given by  
\[
e'_{(p,k)} - e_{(p,k)} = 
- \alpha_p \left(\frac{1}{\beta} - 1\right)  a_p  
+ \alpha_p  
\left( \frac{1}{\beta} - 1 \right) b_p 
 = \frac{\alpha_p (1-\beta)}{\beta} (-a_p + b_p) . 
\] 
Note that $\frac{1}{\alpha_p} > 1$ and $\frac{\alpha_p}{\beta} \leq 1$ for all $1 \leq p \leq k-1$.  
Doing exactly the same for the block $(k,p)$, we obtain the error of the solution $uv^T$, 
\[
e_{(k,p)} 
= 
c_p \left(\frac{1}{\alpha_p}-1 \right) + d_p \left(\frac{1}{\alpha_p}+1 \right), 
\]
and the error of the solution $u'v'^T$, 
\[
e'_{(k,p)} = 
c_p \left(\frac{\beta}{\alpha_p} - 1 \right) 
+ 
d_p \left( \frac{\beta}{\alpha_p} + 1 \right), 
\]
so that  
\[
e'_{(k,p)} - e_{(k,p)} = 
- \frac{1}{\alpha_p} \left( 1 - \beta  \right)  c_p  
- \frac{1}{\alpha_p} \left( 1 - \beta  \right)  d_p  
= - \frac{1 - \beta}{\alpha_p} (c_p + d_p).  
\] 

Finally, the total difference between the objective function of $(u,v)$ and $(u',v')$ is given by 
\[
\delta^{(1)} 
\; = \; 
||A-u'v'^T||_1 - ||A-uv^T||_1 
\; = \; 
\sum_{p = 1}^{k-1} \delta^{(1)}_p 
\; = \;  
(1-\beta) \sum_{p = 1}^{k-1}  
\left( 
- \frac{\alpha_p}{\beta} a_p  
+ \frac{\alpha_p}{\beta} b_p 
-  \frac{1}{\alpha_p} c_p 
- \frac{1}{\alpha_p} d_p 
\right),  
\] 
and is nonnegative since $(u,v)$ is an optimal solution. 

\paragraph{Move 2.} We divide the entries in $u$ different from $\pm 1$ by $-\alpha_{k-1}$ to get $u''$ and multiply the entries in $v$ different from $\pm 1$ by $-\alpha_{k-1}$ to get $v''$. Again, only the blocks at position $(k,p)$ and $(p,k)$ for $1 \leq p \leq k-1$ are affected by these modifications and we can, following exactly the same procedure as for the first move, compute the difference between the objective function value of $(u'',v'')$ and $(u,v)$. The error of the solution $u''v''^T$ for the block $(p,k)$ is given by 
\[
e''_{(p,k)} = 
a_p \left(1+\frac{\alpha_p}{\beta}\right) 
+ 
b_p \left( 1-\frac{\alpha_p}{\beta} \right), 
\]
so that 
\[
e''_{(p,k)} - e_{(p,k)} = 
 \alpha_p \left( \frac{1}{\beta} + 1 \right)  a_p  
- \alpha_p  \left( \frac{1}{\beta} + 1 \right) b_p . 
\]  
For the block $(k,p)$, we obtain the error of the solution $u''v''^T$, 
\[
e''_{(k,p)} = 
c_p   \left( \frac{\beta}{\alpha_p} + 1 \right) 
+ 
d_p \left(\frac{\beta}{\alpha_p} - 1 \right), 
\]
so that  
\[
e''_{(k,p)} - e_{(k,p)} = 
c_p \left( 2 + \frac{\beta}{\alpha_p} - \frac{1}{\alpha_p} \right) 
+ 
d_p \left(\frac{\beta}{\alpha_p} - 2 - \frac{1}{\alpha_p} \right) .   
\]  
Finally, 
\[
\delta^{(2)} 
= \sum_{p = 1}^{k-1}  
\left( 
\alpha_p \left( \frac{1+\beta}{\beta}  \right)  a_p  
- \alpha_p  \left( \frac{1+\beta}{\beta} \right) b_p 
+ \left( 2 + \frac{\beta}{\alpha_p} - \frac{1}{\alpha_p} \right)  c_p 
- \left( 2 + \frac{1}{\alpha_p} - \frac{\beta}{\alpha_p}  \right) d_p 
\right). 
\] 

\paragraph{Combining Move 1 and Move 2} Now, recall that, by optimality of $(u,v)$, $\delta^{(1)} \geq 0$ and $\delta^{(2)} \geq 0$. 
Let us compute the following nonnegative linear combination of $\delta^{(1)}$ and $\delta^{(2)}$: 
\begin{align*} 
0 \; \; \leq  \; \; \frac{1-\beta}{1+\beta} \; \delta^{(1)}  \;  \; + \; \; \delta^{(2)} 
& \quad = \quad  \sum_{p = 1}^{k-1}  
\left( 
-\frac{\alpha_p (\beta+1)}{\beta} + \frac{\alpha_p (\beta+1)}{\beta}  
\right)  
a_p  \\
& \qquad \qquad  + 
\left( 
\frac{\alpha_p (\beta+1)}{\beta} - \frac{\alpha_p (1+\beta)}{\beta}  
\right) 
b_p \\
& \qquad \qquad  + 
\left( 
 \frac{-(1+\beta)}{\alpha_p}  + 2 + \frac{\beta}{\alpha_p} - \frac{1}{\alpha_p} 
\right) 
c_p \\
& \qquad \qquad  + 
\left( 
 \frac{-(1+\beta)}{\alpha_p}  - 2 - \frac{1}{\alpha_p} + \frac{\beta}{\alpha_p}
\right) 
d_p \\ 
& \quad  =  \quad  - \sum_{p = 1}^{k-1} \left(  \frac{2(1-\alpha_p)}{\alpha_p} c_p 
+  \frac{2(1+\alpha_p)}{\alpha_p}  d_p \right) .  
\end{align*}
For all $1 \leq p \leq k-1$, the coefficients for $c_p$ and $d_p$ are positive which implies that $c_p = d_p = 0$ for all $1 \leq p \leq k-1$. In other words, all the $(k,p)$-blocks for $1 \leq p \leq k-1$ are empty: 
this is only possible if $v$ has entries only in $\{-1,+1\}$ hence $u$ also has all its entries in $\{-1,+1\}$. 
\end{proof}

\begin{theorem} \label{th3} 
Rank-one $\ell_1$-LRA~\eqref{l1LRA} is NP-hard. 
\end{theorem}
\begin{proof}  
This follows from Theorem~\ref{th1}, Lemma~\ref{lem3} and Theorem~\ref{th2}.  
\end{proof}  

\begin{remark}[Local minima of $\ell_1$-LRA]  \label{rem2}
Initially, we thought that any local minimum of $\ell_1$-LRA~\eqref{l1LRA} of a $\{-1,+1\}$ matrix can be assumed w.l.o.g.\@ to have entries in $\{-1,+1\}$. However, this is not always true. Here is a counter example: for 
\[
A = \left( 
\begin{array}{cccccc}
    $ $$ $ 1 &   $ $$ $ 1 &   $ $$ $ 1  &  $ $$ $ 1  &  $ $$ $ 1   & $ $$ $ 1 \\
    $ $$ $ 1 &   $ $$ $ 1 &   $ $$ $ 1  &  $ $$ $ 1  &  $ $$ $ 1   & $ $$ $ 1 \\
    $ $$ $ 1 &    -1 &    -1  &   -1  &  $ $$ $ 1   & $ $$ $ 1 \\
     -1 &   $ $$ $ 1 &    -1   &  -1  &  $ $$ $ 1   & $ $$ $ 1 \\
     -1 &    -1 &   $ $$ $ 1   &  -1  &  $ $$ $ 1   & $ $$ $ 1 \\
     -1 &    -1 &    -1  &  $ $$ $ 1  &  $ $$ $ 1   & $ $$ $ 1
\end{array} \right) 
\] 
the solution  
\[
u = [1, 1, x, x, x, x]^T, \; v = [1, 1, 1, 1, 1/x, 1/x]^T
\]
is a stationary point of $\ell_1$-LRA for any $0.5 < x < 1$, 
and is a local minimum for $x = \sqrt{2}/2$ (there is a segment of stationary points with a local minimum in its interior) with error 23.3. 
Using the `Move~2' from the proof of Theorem~\ref{th2}, we obtain 
\[
u = [1, 1, -1, -1, -1, -1]^T, \; v = [1, 1, 1, 1, -1, -1]^T
\] 
which is an optimal solution with error 16 (with 8 mismatches).  
\end{remark}

Theorem~\ref{th2} also has practical implications: in fact, the difficult combinatorial problem~\eqref{inf1norm} of computing the norm $||.||_{\infty \rightarrow 1}$ has a continuous characterization given by $\ell_1$-LRA~\eqref{l1LRA}. 
Therefore, a nice and simple heuristic for computing the norm $||.||_{\infty \rightarrow 1}$ is to use any (iterative) nonlinear optimization scheme for $\ell_1$-LRA (see the introduction). 
Moreover, we have a good initial candidate: the solution of $\ell_2$-LRA that can be computed efficiently via the truncated singular value decomposition. The same continuous characterization can be used to solve other closely related combinatorial problems such as the densest bipartite subgraph problem described in Section~\ref{sec2}.  
We have implemented a simple cyclic coordinate descent method in Matlab for $\ell_1$-LRA, as described in~\cite{KK03}. 
The code is available from \url{https://sites.google.com/site/nicolasgillis/code}, and also contain the matrices from remark~\ref{rem2} and example~\ref{ex1}.

 \section{Higher-rank matrix approximations} \label{sec5}

It is easy to generalize our rank-one NP-hardness results to higher ranks. 
In fact, the same construction as in~\cite[Theorem~3]{GG09} can be used. 
For example, for the rank-one BMF problem, instead of considering the binary input matrix $M$ as in the rank-one case, 
we consider the input matrix 
\[ 
A = 
\left( \begin{array}{ccccc} 
M & 0 & & \dots & 0 \\ 
0 & M & & \dots & 0 \\
 \vdots &  \vdots &  & \ddots &  \vdots \\
0 & 0 & & \dots & M \\ 
 \end{array} \right) = \diag(M,r) . 
\] 
The idea is that an optimal rank-$r$ approximation of $A$ is constituted of rank-one factors that do not have nonzero entries in more than one diagonal block (otherwise the solution can be improved). 
Hence an optimal solution for rank-$r$ BMF of $A$ leads to a combination of optimal rank-one approximations of $M$, one for each block; see the proof in~\cite[Theorem~3]{GG09} for more details.  

Note however that using this construction, the ratio between the factorization rank $r$ and the dimension of the input matrix remains unchanged. Moreover, $\ell_1$-LRA with $r = \min(m,n)-1$ can be solved in polynomial time using linear programming~\cite{BD09, BDB13}. Therefore, it remains an open question whether $\ell_1$-LRA is NP-hard for different values of $r$ depending on $m$ and $n$.   
For example, is $\ell_1$-LRA a difficult problem for $r = \min(m,n)-2$\,?

\section{Conclusion and Future Work}

The main results of this paper are 
\begin{itemize}

\item The equivalence between robust PCA, rank-one $\ell_0$ low-rank matrix approximation, rank-one binary matrix factorization, a particular densest bipartite subgraph problem, the computation of the cut norm and the norm $||.||_{\infty \rightarrow 1}$ of $\{-1,+1\}$ matrices, and the rank-one discrete basis problem.

\item The proof that the optimal rank-one solution of $\ell_1$-LRA of $\{-1,+1\}$ matrices can be assumed without loss of generality to be a  $\{-1,+1\}$ matrix (Theorem~\ref{th2}).   

\item The NP-hardness of the computation of the norm $||.||_{\infty \rightarrow 1}$ of $\{-1,+1\}$ matrices (Theorem~\ref{th1}) which allowed us to prove NP-hardness of all the problems listed above. 

\end{itemize} 

After the publication of an earlier of this paper, Song, Woodruff and Zhong proposed several approximation algorithms for $\ell_1$-LRA~\cite{SWZ17}, addressing the second part of the open question 2 in~\cite{W14}. In particular, they showed that 
it is possible to achieve an approximation factor 
$\alpha = (\log n) \cdot \text{poly}(r)$ in $\text{nnz}(M)+(m+n) \text{poly}(r)$ time, where nnz$(M)$ denotes the number of non-zero entries of $M$.  If $r$ is constant, they further improve the approximation ratio to $O(1)$ with a poly($mn$)-time algorithm. 
They also discuss the extension of their results to $\ell_p$-LRA for $1 < p < 2$, and to other related problems, where $\ell_p$-LRA is defined as 
\[ 
\min_{U \in \mathbb{R}^{p \times r}, V \in \mathbb{R}^{r \times n}} \; 
||M-UV||_p = \sum_{i,j} |M_{ij} - (UV)_{ij}|^p. 
\]

Directions for further research include the study of the complexity of $\ell_p$-LRA,  for $p\notin \{0,1\}$, except for $p=2$ that can be solved in polynomial time via the singular value decomposition. 
We expect that the techniques introduced in section~\ref{sec4} can be extended to show NP-hardness of $\ell_p$-LRA for $p$ in $(0,1)$ (because optimal solutions have the same structure as for the $\ell_1$-norm). 
It would be particularly interesting to know whether $\ell_{\infty}$-LRA is NP-hard, in particular in the rank-one case. In fact, it was shown to be NP-hard for $r = \min(m,n)-1$ in \cite{PR93} (although the problem is stated in a slightly different way).

\section*{Acknowledgements} 

We thank anonymous reviewers for their insightful comments which helped improve the paper. 
N.~Gillis acknowledges the support 
by the F.R.S.-FNRS (incentive grant for scientific research no F.4501.16) and 
by the ERC (starting grant no 679515). 
S.~Vavasis's work is supported in part by a Discovery Grant from NSERC (Natural Sciences and Engineering Research Council) of Canada and the U.S.\@ Air Force office of Scientific Research.

\newpage 

\small

\bibliographystyle{spmpsci}
\bibliography{robustPCAbib}

\begin{thebibliography}{10}
\providecommand{\url}[1]{{#1}}
\providecommand{\urlprefix}{URL }
\expandafter\ifx\csname urlstyle\endcsname\relax
  \providecommand{\doi}[1]{DOI~\discretionary{}{}{}#1}\else
  \providecommand{\doi}{DOI~\discretionary{}{}{}\begingroup
  \urlstyle{rm}\Url}\fi

\bibitem{AN06}
Alon, N., Naor, A.: Approximating the cut-norm via {G}rothendieck's inequality.
\newblock SIAM J. on Computing \textbf{35}(4), 787--803 (2006)

\bibitem{AV11}
Ames, B., Vavasis, S.: Nuclear norm minimization for the planted clique and
  biclique problems.
\newblock Mathematical programming \textbf{129}(1), 69--89 (2011)

\bibitem{AHI02}
Asahiro, Y., Hassin, R., Iwama, K.: Complexity of finding dense subgraphs.
\newblock Discrete Applied Mathematics \textbf{121}(1), 15--26 (2002)

\bibitem{BD09}
Brooks, J., Dul\'a, J.: The {L1}-norm best-fit hyperplane problem.
\newblock Applied Mathematics Letters \textbf{26}(1), 51--55 (2013)

\bibitem{BDB13}
Brooks, J., Dul\'a, J., Boone, E.: A pure {$L_1$}-norm principal component
  analysis.
\newblock Computational Statistics \& Data Analysis \textbf{61}, 83--98 (2013)

\bibitem{BS71}
Brown, T., Spencer, J.: Minimization of $\pm$1 matrices under line shifts.
\newblock In: Colloquium Mathematicae, vol.~23, pp. 165--171. Institute of
  Mathematics Polish Academy of Sciences (1971)

\bibitem{CLM11}
Cand{\`e}s, E., Li, X., Ma, Y., Wright, J.: Robust principal component
  analysis?
\newblock Journal of the ACM \textbf{58}(3), 11 (2011)

\bibitem{CSPW11}
Chandrasekaran, V., Sanghavi, S., Parrilo, P., Willsky, A.: Rank-sparsity
  incoherence for matrix decomposition.
\newblock SIAM J. on Optimization \textbf{21}(2), 572--596 (2011)

\bibitem{CK12}
Chi, E., Kolda, T.: On tensors, sparsity, and nonnegative factorizations.
\newblock SIAM J. Matrix Anal. \& Appl. \textbf{33}(4), 1272--1299 (2012)

\bibitem{CW15}
Clarkson, K., Woodruff, D.: Input sparsity and hardness for robust subspace
  approximation.
\newblock In: 56th Annual IEEE Symposium on Foundations of Computer Science
  (FOCS 2015) (2015)

\bibitem{DV13}
Doan, X., Vavasis, S.: Finding approximately rank-one submatrices with the
  nuclear norm and $\ell_1$-norm.
\newblock SIAM J. on Optimization \textbf{23}(4), 2502--2540 (2013)

\bibitem{EVD10}
Eriksson, A., Van Den~Hengel, A.: Efficient computation of robust low-rank
  matrix approximations in the presence of missing data using the $l_1$ norm.
\newblock In: IEEE Conf. on Computer Vision and Pattern Recognition (CVPR~'10),
  pp. 771--778 (2010)

\bibitem{FK11}
Fiorini, S., Kaibel, V., Pashkovich, K., Theis, D.: Combinatorial bounds on
  nonnegative rank and extended formulations.
\newblock Discrete Mathematics \textbf{313}(1), 67--83 (2013)

\bibitem{FK99}
Frieze, A., Kannan, R.: Quick approximation to matrices and applications.
\newblock Combinatorica \textbf{19}(2), 175--220 (1999)

\bibitem{GZ79}
Gabriel, K., Zamir, S.: Lower rank approximation of matrices by least squares
  with any choice of weights.
\newblock Technometrics \textbf{21}(4), 489--498 (1979)

\bibitem{GG09}
Gillis, N., Glineur, F.: Using underapproximations for sparse nonnegative
  matrix factorization.
\newblock Pattern Recognition \textbf{43}(4), 1676--1687 (2010)

\bibitem{GG10c}
Gillis, N., Glineur, F.: Low-rank matrix approximation with weights or missing
  data is {NP}-hard.
\newblock SIAM J. Matrix Anal. \& Appl. \textbf{32}(4), 1149--1165 (2011)

\bibitem{GV96}
Golub, G., Van~Loan, C.: Matrix Computation, 3rd Edition.
\newblock The Johns Hopkins University Press Baltimore (1996)

\bibitem{GRS12}
Guruswami, V., Raghavendra, P., Saket, R., , Wu, Y.: Bypassing ugc from some
  optimal geometric inapproximability results.
\newblock In: In Proceedings of the twenty-third annual ACM-SIAM symposium on
  Discrete Algorithms (SODA 2012), pp. 699--717 (2012)

\bibitem{KV99}
Kannan, R., Vinay, V.: {Analyzing the structure of large graphs} (1999).
\newblock Http://www.cs.yale.edu/homes/kannan/Papers/webgraph.pdf

\bibitem{KK03}
Ke, Q., Kanade, T.: Robust subspace computation using {L1} norm (2003).
\newblock
  \url{http://www.cs.cmu.edu/afs/.cs.cmu.edu/Web/People/ke/publications/CMU-CS-03-172.pdf.}

\bibitem{KK05}
Ke, Q., Kanade, T.: Robust $l_1$ norm factorization in the presence of outliers
  and missing data by alternative convex programming.
\newblock In: IEEE Conf. on Computer Vision and Pattern Recognition (CVPR~'05),
  pp. 739--746 (2005)

\bibitem{K06}
Khot, S.: Ruling out {PTAS} for graph min-bisection, dense k-subgraph, and
  bipartite clique.
\newblock SIAM J. on Computing \textbf{36}(4), 1025--1071 (2006)

\bibitem{KS09}
Khuller, S., Saha, B.: On finding dense subgraphs.
\newblock In: S.~Albers, A.~Marchetti-Spaccamela, Y.~Matias, S.~Nikoletseas,
  W.~Thomas (eds.) Automata, Languages and Programming, \emph{Lecture Notes in
  Computer Science}, vol. 5555, pp. 597--608 (2009)

\bibitem{KBV09}
Koren, Y., Bell, R., Volinsky, C.: Matrix factorization techniques for
  recommender systems.
\newblock IEEE Computer \textbf{42}(8), 30--37 (2009)

\bibitem{K08}
Kwak, N.: Principal component analysis based on {L1}-norm maximization.
\newblock IEEE Trans. on Pattern Analysis and Machine Intelligence
  \textbf{30}(9), 1672--1680 (2008)

\bibitem{MU13}
Markovsky, I., Usevich, K.: Structured low-rank approximation with missing
  data.
\newblock SIAM J. Matrix Anal. \& Appl. \textbf{34}(2), 814--830 (2013)

\bibitem{MMG08}
Miettinen, P., Mielikainen, T., Gionis, A., Das, G., Mannila, H.: The discrete
  basis problem.
\newblock IEEE Transactions on Knowledge and Data Engineering \textbf{20}(10),
  1348--1362 (2008)

\bibitem{MGT15}
Mirisaee, S., Gaussier, E., Termier, A.: Improved local search for binary
  matrix factorization.
\newblock In: AAAI Conference on Artificial Intelligence, pp. 1198--1204 (2015)

\bibitem{NNSA14}
Netrapalli, P., Niranjan, U., Sanghavi, S., Anandkumar, A.: Non-convex robust
  {PCA}.
\newblock In: Advances in neural information processing systems (NIPS), pp.
  2080--2088 (2014)

\bibitem{PR93}
Poljak, S., Rohn, J.: Checking robust nonsingularity is {NP}-hard.
\newblock Mathematics of Control, Signals and Systems \textbf{6}(1), 1--9
  (1993)

\bibitem{CVL14}
Qiu, C., Vaswani, N., Lois, B., Hogben, L.: Recursive robust {PCA} or recursive
  sparse recovery in large but structured noise.
\newblock IEEE Trans. on Information Theory \textbf{60}(8), 5007--5039 (2014)

\bibitem{R2000}
Rohn, J.: Computing the norm {$||A||_{\infty, 1}$} is {NP}-hard.
\newblock Linear and Multilinear Algebra \textbf{47}(3), 195--204 (2000)

\bibitem{SJY09}
Shen, B.H., Ji, S., Ye, J.: Mining discrete patterns via binary matrix
  factorization.
\newblock In: Proc. of the 15th ACM SIGKDD Int. Conf. on Knowledge Discovery
  and Data Mining, KDD '09, pp. 757--766 (2009)

\bibitem{SIR95}
Shum, H., Ikeuchi, K., Reddy, R.: {Principal component analysis with missing
  data and its application to polyhedral object modeling}.
\newblock IEEE Trans. on Pattern Analysis and Machine Intelligence
  \textbf{17}(9), 854--867 (1995)

\bibitem{SWZ17}
Song, Z., Woodruff, D., Zhong, P.: Low rank approximation with entrywise
  $\ell_1$-norm error.
\newblock In: 49th Annual ACM SIGACT Symposium on the Theory of Computing (STOC
  2017) (2017)

\bibitem{UM14}
Usevich, K., Markovsky, I.: Optimization on a {G}rassmann manifold with
  application to system identification.
\newblock Automatica \textbf{50}, 1656--1662 (2014)

\bibitem{W14}
Woodruff, D.: Sketching as a tool for numerical linear algebra.
\newblock Foundations and Trends$^{\mbox{\scriptsize{\textregistered}}}$ in
  Theoretical Computer Science \textbf{10}(1-2), 1--157 (2014)

\bibitem{WG09}
Wright, J., Ganesh, A., Rao, S., Peng, Y., Ma, Y.: Robust principal component
  analysis: Exact recovery of corrupted low-rank matrices via convex
  optimization.
\newblock In: Advances in neural information processing systems (NIPS), pp.
  2080--2088 (2009)

\bibitem{XCS12}
Xu, H., Caramanis, C., Sanghavi, S.: Robust {PCA} via outlier pursuit.
\newblock IEEE Trans. on Information Theory \textbf{58}(5), 3047--3064 (2012)

\bibitem{SLD10}
Zhang, Z.Y., Li, T., Ding, C., Ren, X.W., Zhang, X.S.: Binary matrix
  factorization for analyzing gene expression data.
\newblock Data Min. Knowl. Disc. \textbf{20}, 28--52 (2010)

\bibitem{ZLS12}
Zheng, Y., Liu, G., Sugimoto, S., Yan, S., Okutomi, M.: Practical low-rank
  matrix approximation under robust {L1}-norm.
\newblock In: IEEE Conf. on Computer Vision and Pattern Recognition (CVPR'~12),
  pp. 1410--1417 (2012)

\end{thebibliography}

\end{document}